\documentclass[conference]{IEEEtran}
\usepackage{times}
\usepackage{multicol}
\usepackage[bookmarks=true]{hyperref}
\usepackage{amsmath}
\usepackage{amssymb}
\usepackage{amsthm}
\usepackage{stmaryrd}
\usepackage{graphicx}
\usepackage{tikz}
\usepackage{setspace}
\usepackage{mathrsfs}
\usepackage{bm}
\usepackage[english]{babel}
\usepackage{blindtext}
\usepackage{epstopdf}
\usepackage{siunitx}
\usepackage{booktabs}

\newtheorem{definition}{Definition}
\newtheorem{proposition}{Proposition}

\newcommand{\comment}[1]{}

\newcommand{\bbR}{\mathbb{R}}

\newcommand{\calA}{{\cal A}}

\newcommand{\calF}{{\cal F}}

\newcommand{\calP}{{\cal P}}

\newcommand{\beg}{\rm{beg}}
\newcommand{\Abeg}{\rm{Abeg}}
\newcommand{\Bbeg}{\rm{Bbeg}}
\newcommand{\Aend}{\rm{Aend}}
\newcommand{\Bend}{\rm{Bend}}
\newcommand{\A}{\rm{A}}
\newcommand{\B}{\rm{B}}
\newcommand{\C}{\rm{C}}
\newcommand{\MaAS}{\rm{MaAS}}
\newcommand{\MiAS}{\rm{MaAS}}

\renewcommand{\vec}[1]{\mathbf{#1}}

\usepackage[ruled, linesnumbered]{algorithm2e}
\let\oldnl\nl
\newcommand{\nonl}{\renewcommand{\nl}{\let\nl\oldnl}}
\SetKwIF{If}{ElseIf}{Else}{if}{ then}{elif}{else}{}%
\SetKwFor{For}{for}{ do}{}%
\SetKwFor{ForEach}{foreach}{ do}{}%
\SetKwInOut{Input}{Input}%
\SetKwInOut{Output}{Output}%
\AlgoDontDisplayBlockMarkers%
\SetAlgoNoEnd%
\SetAlgoNoLine%
\DontPrintSemicolon

\SetKwFunction{bridge}{bridge}
\SetKwFunction{extend}{extend}
\SetKwFunction{extendbw}{extend}
\SetKwFunction{integrate}{integrate} 

\begin{document}
\title{\LARGE On the Structure of the Time-Optimal Path
  Parameterization Problem with Third-Order Constraints}

\author{Hung Pham, Quang-Cuong Pham\\ \thanks{Hung Pham and
    Quang-Cuong Pham are with Air Traffic Management Research
    Institute (ATMRI) and Singapore Centre for 3D Printing (SG3DP),
    School of Mechanical and Aerospace Engineering, Nanyang
    Technological University, Singapore. This work was partially
    supported by grant ATMRI:2014-R6-PHAM awarded by NTU and the Civil
    Aviation Authority of Singapore to the last author.}  }

\date{\today} 
\maketitle
\thispagestyle{empty}
\pagestyle{empty} 

\begin{abstract}
  Finding the Time-Optimal Parameterization of a Path (TOPP) subject
  to second-order constraints (e.g. acceleration, torque, contact
  stability, etc.) is an important and well-studied problem in
  robotics. In comparison, TOPP subject to third-order constraints
  (e.g. jerk, torque rate, etc.) has received far less attention and
  remains largely open. In this paper, we investigate the structure of
  the TOPP problem with third-order constraints. In particular, we
  identify two major difficulties: (i) how to smoothly connect optimal
  profiles, and (ii) how to address singularities, which stop profile
  integration prematurely. We propose a new algorithm, TOPP3, which
  addresses these two difficulties and thereby constitutes an
  important milestone towards an efficient computational solution to
  TOPP with third-order constraints.
\end{abstract}

\section{Introduction}
Given a robot and a smooth path in the robot's configuration space,
the problem of finding the Time-Optimal Parameterization of that Path
(TOPP) with second-order constraints (e.g. bounds on acceleration,
torques, contact stability) is an important and well-studied in
robotics. An efficient algorithm to solve this problem was first
proposed in the 1980's~\cite{BobX85ijrr,SM86tac} and has been
continuously perfected since then, see~\cite{Pha14tro} for a
historical review. The algorithm has also been extended to handle a
wide range of problems, from manipulators subject to torque
bounds~\cite{BobX85ijrr,SL92jdsmc} to vehicles or legged robots
subject to balance
constraints~\cite{SG91tra,PN12humanoids,PS15tmech,caron2015leveraging},
to kinodynamic motion planning~\cite{pham2017avp}, etc.

According to Pontryagin's Maximum Principle, time-optimal trajectories
are bang-bang, which implies instantaneous switches in the
second-order quantities that are constrained. For instance,
time-optimal trajectories with acceleration constraints will involve
acceleration switches, which in turn implies infinite jerk. This is
one of the drawbacks of TOPP with second-order constraints.

To address this issue, one can consider TOPP with third-order
constraints. In many industrial applications, constraining third-order
quantities such as jerk, torque rate or force rate is also part of the
problem definition~\cite{tarkiainen1993time}. For instance, a
hydraulic actuator exerts forces by forcing oil to travel through its
piston and gets compressed, which results in the actuating force rate
being restricted. As another example, DC electric motors have bounds
on input voltages, which translate directly to torque rate
constraints.

\subsection{Related works} 
\label{sub:related_works}

While TOPP with second-order constraints can essentially be considered
as solved~\cite{Pha14tro}, the structure of TOPP with third-order
constraints is much less well understood. In the sequel, we survey
some of the attempts to address TOPP with third-order constraints.

\subsubsection{Numerical integration methods}

TOPP with second-order constraints can be efficiently solved by the
numerical integration method: from Pontryagin's Maximum Principle, the
optimal velocity profile in the ($s,\dot s$) plane [$s(t)$ denotes the
position on the path at a given time instant $t$] is ``bang-bang'' and
can thus be found by integrating successively the maximum and minimum
accelerations~$\ddot s$. In~\cite{tarkiainen1993time}, Tarkiainen and
Shiller pioneered the extension this approach to third-order
constraints. Here, one integrates profiles in the \emph{3D space}
($s,\dot s, \ddot s$), following successively maximum, minimum and
maximum jerk~$\dddot s$. As discussed later in
Section~\ref{sec:prelim}, there are two main difficulties: (i)
contrary to the 2D case, minimum and maximum profiles in the
($s,\dot s, \ddot s$) space do not generally intersect each other,
(ii) there are \emph{singularities} that stop profile integration
prematurely. To address (i), the authors proposed to connect the
initial maximum jerk profile to the final maximum jerk profile by (a)
stepping along the initial profile, (b) at each step, integrate a
minimum jerk profile and check whether that profile connects with the
final profile. This procedure corresponds to a Single Shooting method,
which is known to be sensitive to initial condition
\cite{stoer2013introduction}.  Regarding~(ii), the authors did not
propose a method to overcome singularities; instead, they suggested to
modify the original path to avoid them. However, this workaround is
questionable as in practice, paths usually contain a large number of
singularities~\cite{Pha14tro}.

In a recent work~\cite{mattmuller2009calculating}, Mattmuller and
Gilser proposed an approximate method to compute near-optimal
profiles: they introduce ``split points'' artificially to ``guide''
the profiles away from singularities.  While their method presents the
advantage of simplicity, it is sub-optimal and does not help
understanding the structure of the true time-optimal solutions.

\subsubsection{Optimization-based methods}

TOPP with second-order constraints can also be solved robustly (albeit
at the expense of computation speed~\cite{Pha14tro}) using convex
optimization~\cite{verscheure2008practical,Hau14ijrr}. While we are
not aware of any result on formulating TOPP with third-order
constraints as a convex optimization problem, there have been a number
of works that \emph{approximate} the constraints to make them
convex. For instance, the authors of~\cite{zhang2013practical}
approximated jerk constraint by linear functions. In
\cite{gasparetto2008technique}, the authors replaced jerk constraints
by squared jerk terms in the objective function. In a recent
work~\cite{singh2015class}, the authors proposed to represent the
profiles by a class of $C^{\infty}$ functions. This representation
ensures the continuities of higher-order derivatives such as jerk,
snap, etc. while still allowing for efficient solutions via convex
optimization.

In~\cite{costantinescu2000smooth,bharathi2015smooth,liusmooth}, the
authors approximated the optimal profiles by a piece-wise polynomial
in the $(s, \dot s)$ plane. The polynomials are controlled via a
finite set of knot points. The path parameterization problem then
becomes equivalent to an optimization problem where the independent
variables are the coordinates of the knot points. However, this formulation
is non-convex and there is no guarantee that it converges to true
time-optimal solutions.

\subsection{Our contributions}

In this paper, we follow the numerical integration approach and
investigate the structure of the time-optimal solutions. Specifically,
our contribution is threefold:
\begin{enumerate}
\item we propose a Multiple Shooting Method (MSM) to smoothly connect
  two maximum jerk profiles by a minimum jerk profile. This method is
  more efficient and stable than the the one proposed
  in~\cite{tarkiainen1993time};
\item we analyze third-order singularities, which arise very
  frequently and systematically stop profile integration, leading to
  algorithm failure. We propose a method to address such
  singularities;
\item based on the above contributions, we implement TOPP3, which
  solves TOPP with third-order constraints efficiently and yields
  true optimal solutions in a number of situations.
\end{enumerate}

For simplicity, this paper focuses on pure third-order
constraints. Taking into account first- and second-order constraints
(e.g. bounds on velocity and acceleration) is possible but will
significantly increase the complexity of the exposition.

The remainder of the paper is organized as follows. In
Section~\ref{sec:prelim}, we introduce the basic notations and
formulate the problem. In Section~\ref{sec:connection}, we introduce
our method to smoothly connect two maximum jerk profiles. In
Section~\ref{sec:singularities}, we propose a solution to the problem
of singularities. In Section~\ref{sec:experiments}, we present the
numerical experiments. Finally, in Section~\ref{sec:discussion}, we
offer some discussions and directions for future work.

\section{Structure of TOPP with third-order constraints}
\label{sec:prelim}

\subsection{Problem setting}

Consider a $n$ dof robotic system whose configuration is a vector
$\vec q \in \bbR^n$. A geometric path $\calP$ is a mapping $\vec
q(s)_{s\in[0, s_\mathrm{end}]}$ from $[0, s_\mathrm{end}]$ to the
configuration space.  A time-parameterization of the path $\calP$ is
an increasing scalar mapping $s(t)_{t \in [0, T]}$. Differentiating
successively $\vec q(s(t))$ with respect to $t$ yields
\begin{equation}
  \label{eq:kinematics-rel}
    \begin{aligned}
\dot    {\vec q} = &  \vec q_{s}    {\dot s} \\ 
\ddot   {\vec q} = &  \vec q_{ss}   {\dot s}^2 + \vec q_s \ddot s \\
\dddot  {\vec q} = &  \vec q_{sss}  {\dot s}^3 + 3 \vec q_{ss} \dot s \ddot s + \vec q_{s} \dddot s,
    \end{aligned}
\end{equation}
where $\dot \Box$ denotes time-derivative.

From here on, we shall refer to the first, second and third
time-derivatives of the path parameter $s$ as velocity, acceleration
and jerk respectively. The time-derivatives of the configuration $\vec
q$ will be called joint velocity, acceleration and jerk.


The boundary conditions of a TOPP with third-order constraints consist
of configuration velocities and accelerations at both the start and
the goal. Combining with equation~(\ref{eq:kinematics-rel}), one can
compute the corresponding initial velocities and accelerations as
\[
  s_0 = 0, \quad
  \dot s_0 = \frac{\|\vec {v}_{\beg}\|}{\|\vec q_s(0)\|}, \quad
  \ddot s_0 = \frac{\|\vec{a}_{\beg} - \vec q_{ss}(0) \dot s_0^2\|}{ \|\vec
    q_s(0)\|},
\]
where $\vec{v}_{\beg}$ and $\vec{a}_{\beg}$ are the initial joint
velocity and acceleration respectively.  The end conditions
$(s_1, \dot s_1, \ddot s_1)$ can be computed similarly.

We consider third-order constraints of the form
\begin{equation}
\label{eq:third}
    \vec a(s) \dddot s + \vec b(s) \dot s \ddot s + \vec c(s) \dot s^3 
    + \vec d(s) \leq 0,
\end{equation}
where $\vec a(s), \vec b(s), \vec c(s), \vec d(s)$ are $m$-dimensional vectors.

As in the second-order case, the bounds~(\ref{eq:third}) can represent
a wide variety of constraints, from direct jerk bounds to bounds on
torque rate or force rate\,\footnote{In fact, torque rate and force
  rate involve an additional term $\dot s \vec e(s)$, but their
  treatment is not fundamentally different from what is presented in
  this paper.}. For instance, direct jerk bounds
$(\vec j_{\min{}} \leq \dddot{\vec q} \leq \vec j_{\max{}})$ can be
accommodated by setting $\vec a$, $\vec b$, $\vec c$, $\vec d$ as
follows
\begin{align*}
  &\vec a(s) = \begin{pmatrix} \vec q_s(s) \\ -\vec
    q_s(s) \end{pmatrix},&
  &\vec b(s) = \begin{pmatrix}
    3\vec q_{ss}(s) \\ -3\vec q_{ss}(s) \end{pmatrix},&\\
  &\vec c(s) = \begin{pmatrix} \vec q_{sss}(s) \\ -\vec
    q_{sss}(s) \end{pmatrix},&
&\vec d(s) = \begin{pmatrix} -\vec
    j_{\max{}} \\ \vec j_{\min{}} \end{pmatrix}.&
\end{align*}

Similar to the classical TOPP algorithm with second order
constraints~\cite{BobX85ijrr,Pha14tro}, one can next define, at any
state $(s, \dot s, \ddot s)$, the minimum jerk
$\gamma(s, \dot s, \ddot s)$ and maximum jerk
$\eta(s, \dot s, \ddot s)$ as follows
\begin{equation}
\label{eq:controls}
\begin{aligned}
  \gamma(s, \dot s, \ddot s):=\max_i\left\{
  \frac{- b_i(s) \dot s \ddot s - c_i(s) \dot s^3 - d_i(s) }{a_i(s)}
  \mid a_i(s) < 0 \right\}, \\
  \eta(s, \dot s, \ddot s):=\min_i\left\{
  \frac{- b_i(s) \dot s \ddot s - c_i(s) \dot s^3 - d_i(s) }{a_i(s)}
  \mid a_i(s) > 0 \right\}. \\
\end{aligned}
\end{equation}
Note that the above definitions neglect the case where some $a_i$ are
zero.  As in the case of second-order constraints, a path position $s$
where at least one of the $a_i(s)$ is zero can trigger a
\emph{singularity}.

A \emph{profile} is is a curve in the $(s, \dot s, \ddot s)$-space
where $s$ is always increasing. A time-parameterization of $\vec q$
corresponds to a profile connecting $(s_0, \dot s_0, \ddot s_0)$ to
$(s_1, \dot s_1, \ddot s_1)$. From the definition of $\gamma$ and
$\eta$, a time-parameterization of $\vec q$ satisfies the
constraints~(\ref{eq:third}) if and only if the jerk $\dddot s$ along
the corresponding profile satisfies $\gamma \leq \dddot s \leq \eta$.

Next, we define a minimum jerk (resp. maximum jerk) profile as a
profile along which $\dddot s = \gamma(s, \dot s, \ddot s)$
(resp. $\dddot s = \eta(s, \dot s, \ddot s)$). From Pontryagin's
Maximum Principle, the time-optimal profile follows successively
maximum and minimum jerk profiles~\cite{tarkiainen1993time}. Finding
which profiles to follow and when to switch between two consecutive
profiles is therefore the fundamental issue underlying TOPP with
third-order constraints.

\subsection{Introducing TOPP3}

If there is no singularity, the optimal profile has a max-min-max
structure~\cite{tarkiainen1993time}. Thus, to find the optimal
profile, one can proceed as follows, see Fig.~\ref{fig:max-min-max}
for an illustration
\begin{enumerate}
\item integrate \emph{forward} the maximum jerk profile (following $\eta$) from
  $(s_0, \dot s_0, \ddot s_0)$;
\item integrate \emph{backward} the maximum jerk profile (following
  $\eta$) from $(s_1, \dot s_1, \ddot s_1)$;
\item find a minimum jerk profile that starts from one point on the
  first profile and ends at one point on the second one.
\end{enumerate}
All profiles are integrated until failure unless otherwise specified.

\begin{figure}[htp]
  \centering
  \includegraphics[width=0.35\textwidth]{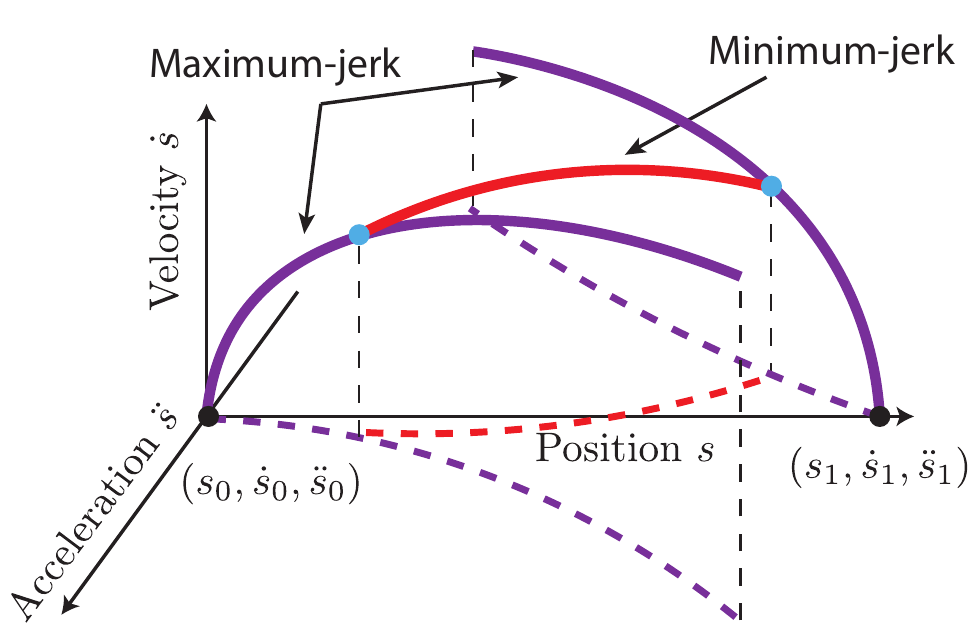}
  \caption{An optimal parametrization with max-min-max structure.}
  \label{fig:max-min-max}
\end{figure}

To find the connecting profile of step (3), Tarkiainen and Shiller
proposed to step along the first profile, integrate following minimum
jerk and check whether it connects to the second profile. This
exhaustive procedure is computationally inefficient and numerically
unstable. We propose, in Section~\ref{sec:connection}, a more
efficient way (\bridge) to perform the connection.

In the presence of singularities ($a_i(s)=0$), the max-min-max
procedure just described will fail, irrespective of the connection
method. This is because the integrated profiles diverge as they
approach a singularity and quickly terminate, see
Fig.~\ref{fig:fields}A. This behavior is the same as in the
second-order case~\cite{Pha14tro}. To address this issue, we propose,
in Section~\ref{sec:singularities}, a method (\extend) that allows
extending profiles \emph{through} singularities. Note that the
structure of the final profile will no longer be max-min-max but
max-min-max-\dots-min-max.

The above discussion can be encapsulated into
Algorithm~\ref{algo:topp3}, which we call \emph{TOPP3}. In the next
sections, we shall discuss in detail the two main components of TOPP3,
namely \bridge and \extend.

\begin{algorithm}[htp]
\label{algo:topp3}
  \caption{TOPP3}
  \Input{Initial and final conditions $(s_0, \dot s_0, \ddot s_0)$,
    $(s_1, \dot s_1, \ddot s_1)$} %
  \Output{The optimal profile $P_\mathrm{opt}$} %
  $P_\mathrm{forward} \leftarrow$ integrate forward from
  $(s_0, \dot s_0, \ddot s_0)$ following maximum jerk until
  termination\; %
  \While{found next singularity}{ $P_\mathrm{forward}\leftarrow$
    \extend{$P_\mathrm{forward}$, singularity}\; }
  $P_\mathrm{backward}\leftarrow$ integrate backward from
  $(s_1, \dot s_1, \ddot s_1)$ following maximum jerk until termination\; %
  \While{found next singularity}{
    $P_\mathrm{backward}\leftarrow$ \extend{$P_\mathrm{backward}$,
      singularity}\; } %
  $P_\mathrm{opt}\leftarrow$
  \bridge{$P_\mathrm{forward}, P_\mathrm{backward}$}\;
\end{algorithm}

\subsection{Remarks}

\subsubsection{Optimality of the connection}

In the absence of singularities, the max-min-max structure described
earlier is only a necessary condition for optimality and not a
sufficient one. Theoretically, to find the true optimal solution, it
suffices to enumerate all candidates having max-min-max structure and
select the one with the shortest time duration.

Next, since the first and last maximum jerk profiles are
well-determined (their starting (resp. ending) conditions are given),
different solutions only differ by the connecting minimum jerk
profile. If there exists one unique minimum jerk profile that can
connect the two maximum jerk profiles, then the corresponding solution
is automatically the optimal one. In practice, we have never
encountered the case when there are more than one possible connecting
minimum jerk profile. Note that Tarkiainen and
Shiller~\cite{tarkiainen1993time} implicitly assumed that there exists
one unique connecting minimum jerk profile.

\subsubsection{Other switching structures}

The TOPP problem with second-order constraints involves three types of
switch points, i.e., points where the optimal profile changes from
minimum acceleration to maximum acceleration and vice-versa
(cf. ~\cite{Pha14tro}): (a) singular (b) discontinuous and (c)
tangent points. These switch points, in general, lie on the Maximum
Velocity Curve (MVC).

We argue that TOPP with third-order constraints similarly involves
singular, discontinuous and tangent switch points, and that those
points, in general, lie on the Maximum and Minimum Acceleration
Surfaces (MaAS and MiAS). 

In Section~\ref{sec:singularities}, we shall discuss the singular
switch points, which are the most commonly encountered and harmful (if
inappropriately treated) switch points. As in the case of second-order
constraints, discontinuous switch points can always be avoided if the
path is sufficiently smooth. Finally, tangent switch points are left for
future work. Let us simply note here that they occur much less frequently
than singular switch points.

Another type of switching structure arises when constraints of
different orders interact. For instance, in the classic TOPP problem,
when an integrated profile hits a first-order constraint (e.g. direct
velocity bound), then it must ``slide'' along the boundary defined by
that first-order constraint, giving rise to a different type of switch
point. We envisage that such switching behavior can also happen when
first, second, and third-order constraints interact. The study of these
interactions is also left for future work.

\section{Connecting profiles using Multiple Shooting}
\label{sec:connection}

Instead of the exhaustive search suggested
in~\cite{tarkiainen1993time}, we propose here a method, termed
\bridge, which is based on Multiple Shooting to find a minimum jerk
profile that can connect two maximum jerk profiles. Specifically,
consider the problem of connecting two maximum jerk profiles $A$ and
$B$. We define a potential solution $\vec x$ as a
$(2N + 4)$-dimensional vector
\[
  \vec x:=[\dot s_0, \ddot s_0,...,\dot s_N, \ddot s_N, s_A, s_B], 
\]
where $\dot s_i, \ddot s_i$, $i\in [0,N]$ are the guessed velocities
and accelerations at the $i$-th point and $(s_A, s_B)$ are the guessed
starting and ending positions on profiles $A$ and $B$ respectively. We
consider a uniform grid, i.e.
\[
    s_i := s_A + \frac{i}{N} (s_B-s_A).
\]

Next, assume that we integrate following minimum jerk from $(s_i,\dot
s_i, \ddot s_i)$ until $s_j$. We define the function $X:\bbR^4
\rightarrow \bbR^2$ as
\[
    X(s_i, \dot s_i, \ddot s_i, s_j):=(\dot s_j, \ddot s_j),
\]
where $\dot s_j, \ddot s_j$ are the corresponding velocity and
acceleration at $s_j$. This allows us to define the defect function by
\begin{equation}
    \label{eq:systems}
    F(\vec x) := 
    \begin{pmatrix}
        X(s_0, \dot s_0, \ddot s_0, s_1) - [\dot s_1, \ddot s_1]^T \\
        X(s_1, \dot s_1, \ddot s_1, s_2) - [\dot s_2, \ddot s_2]^T  \\
        ...\\
        X(s_{N-1}, \dot s_{N-1}, \ddot s_{N-1}, s_{N}) - [\dot s_{N}, \ddot s_{N}]^T \\
        r_{\A}(s_{\A}) - [\dot s_0, \ddot s_0]^T \\
        r_{\B}(s_{\B}) - [\dot s_N, \ddot s_N]^T\\        
    \end{pmatrix},
\end{equation}
where $r_{\A}(s_{\A})$ (resp. $r_{\B}(s_{\B})$) are the velocity and
acceleration on profile $\A$ (resp. $\B$) at position $s_{\A}$
(resp. $s_{\B}$).

The connection problem can now be formulated as
\begin{equation}
  \label{eq:MSM}
  \begin{aligned}
    \mathrm{solve\quad} & F(\vec x) = 0\\
    \mathrm{subject\;to\quad} &
    s_{\Abeg} \leq s_{\A} \leq s_{\Aend}, \\
    & s_{\Bbeg} \leq s_{\B} \leq s_{\Bend},
  \end{aligned}
\end{equation}
where $(s_{\Abeg}, s_{\Aend})$ denote positions of profile
$\A$'s end points (resp. profile $\B$).

To solve (\ref{eq:MSM}), we employ the Newton method. Although the
problem is non-linear and non-convex, the algorithm still converges
very quickly to a solution. Fig.~\ref{fig:multiple-shooting} shows
a particular instance where a solution can be found in 4 iterations.

\begin{figure}[tb]
    \centering
    \includegraphics[width=0.4\textwidth]{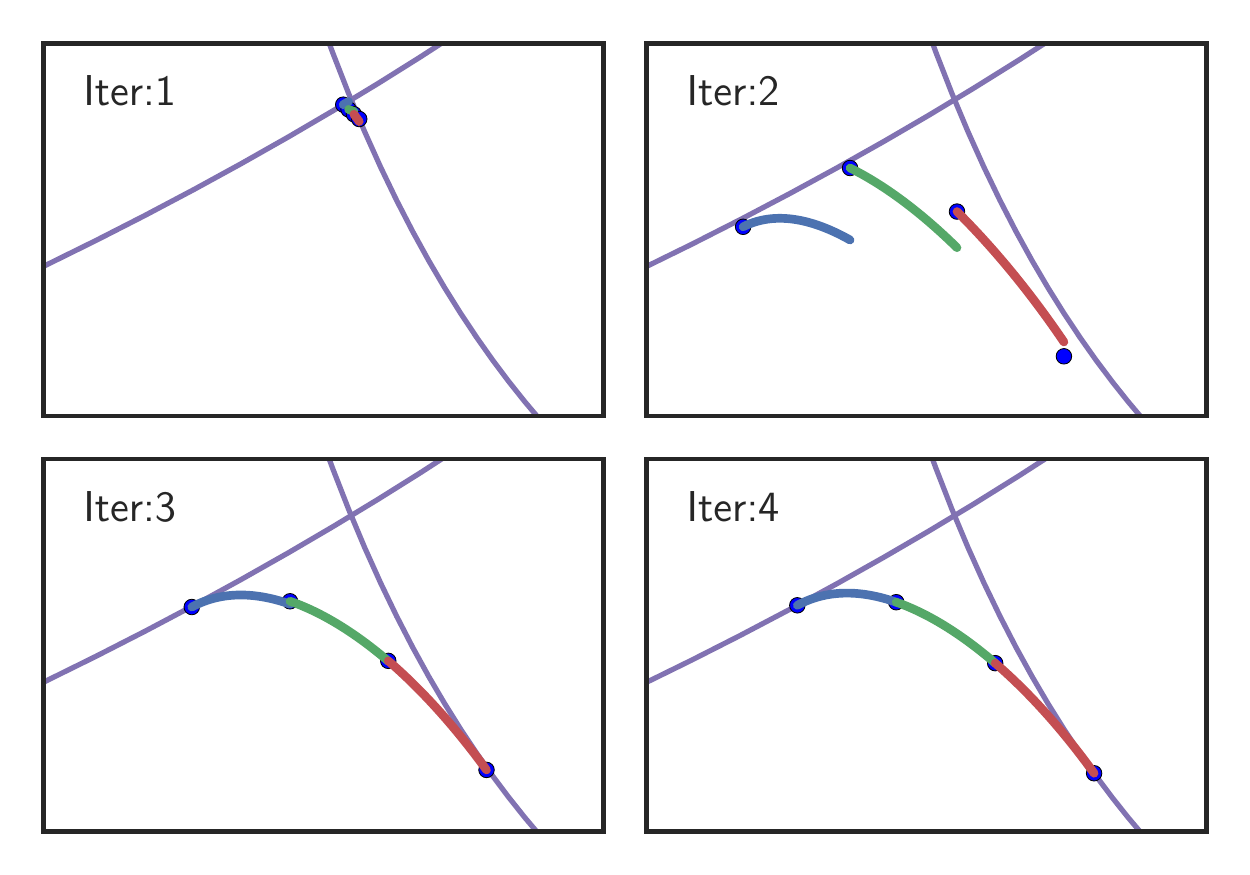}
    \caption{Using MSM to find a minimum jerk profile connecting two
      maximum jerk profiles (purple). Note that the profiles are in 3D
      but projected to 2D for convenience. The number of segments is
      $N=3$. The algorithm took 4 iterations to converge.}
    \label{fig:multiple-shooting}
\end{figure}

Regarding computational cost, we found that it correlates to the
magnitude of jerk bounds. For instance, at $\SI{1000}{rad s^{-3}}$ a
\bridge function call takes only $\SI{10}{ms}$ while at $\SI{100}{rad
  s^{-3}}$, it is approximately $\SI{200}{ms}$. How jerk bounds
precisely affect computation time is left for future investigations.

\section{Characterizing and addressing singularities}
\label{sec:singularities}

Singularities in the third-order case are very similar to those in the
second-order case: (i) they arise at positions $s$ where the minimum
and maximum jerk $\gamma$ and $\eta$ cannot be properly defined
because of the division by $a_k(s)=0$ for some $k$; (ii) they cause
profiles to terminate prematurely, causing algorithm failure (see
Fig.~\ref{fig:fields}). In the sequel, we discuss how to characterize
third-order singularities and how to address singularities, taking
much inspiration from the second-order case~\cite{Pha14tro}.

\begin{figure}[htp]
    \centering
\begin{tikzpicture}
    \node[anchor=south west,inner sep=0] at (0,5)
         {\includegraphics[width=0.35\textwidth]{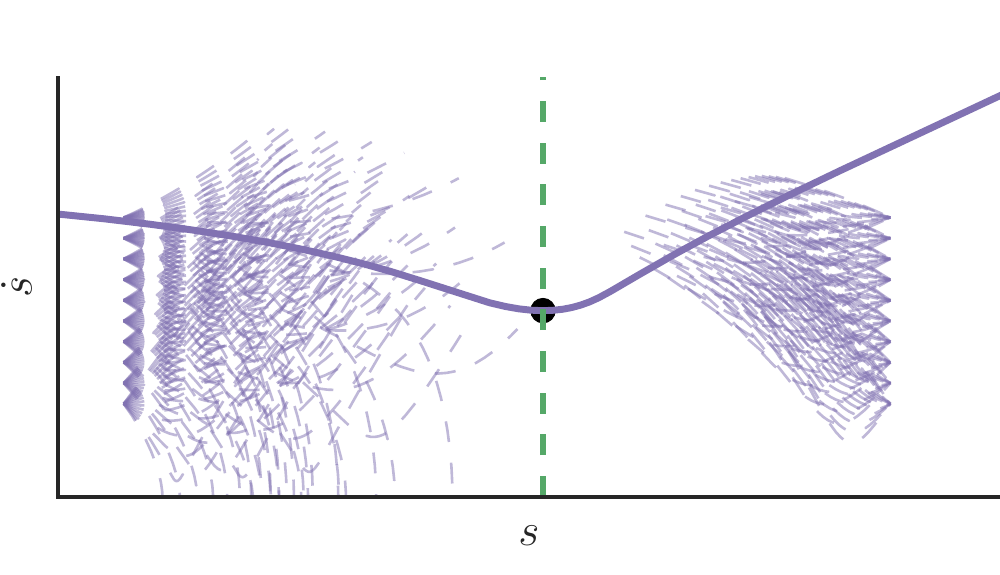}}; 
    \node[anchor=south west,inner sep=0] at (0,0)
         {\includegraphics[width=0.35\textwidth]{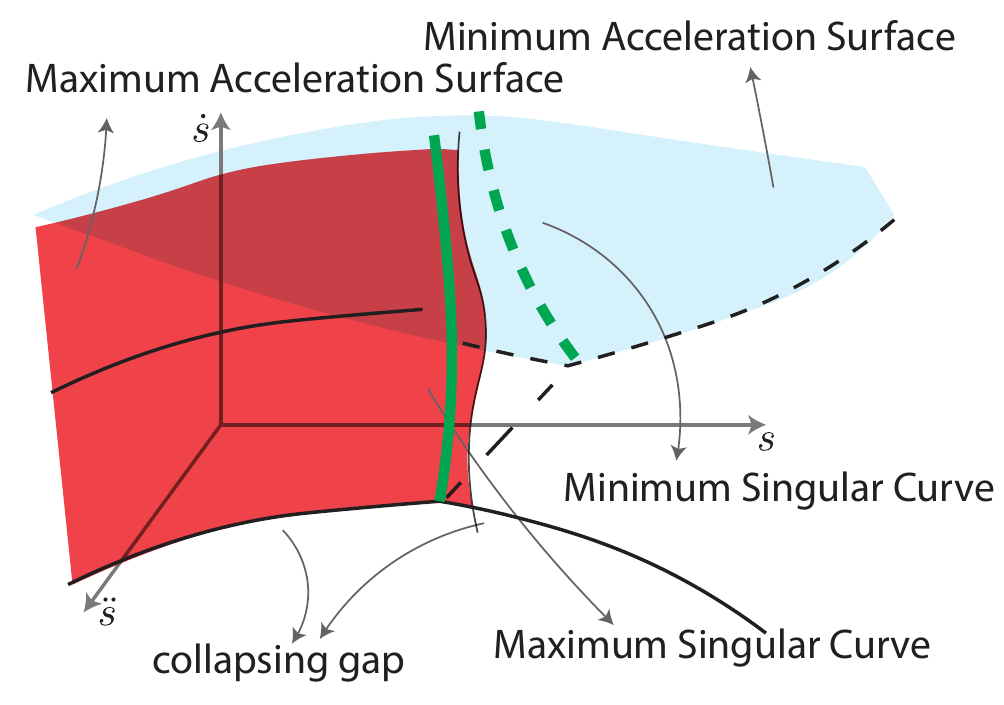}}; 
    \node at (3, 8.4) {\textbf{A}};
    \node at (3, 4.7) {\textbf{B}};
\end{tikzpicture}
\caption{ \textbf{A}: Profiles approaching singularities (purple
  dashed lines) diverge and terminate early. The diverging directions
  include both vertical (along~$\dot s$ axis) and lateral
  (along~$\ddot s$ axis, which is not shown). However, profiles
  (purple solid lines) starting from a point lying between the
  singular curves (green dashed lines) are not affected by the
  singularities.  \textbf{B}: Singularities consist of maximum or
  minimum singular curves lying on the maximum or minimum acceleration
  surfaces (MaAS/MiAS) respectively.  }
    \label{fig:fields}
\end{figure}




\subsection{Maximum and Minimum Acceleration Surfaces}

In the second-order case, singularities mostly appear on the Maximum
Velocity Curve
(MVC)~\cite{Pha14tro,pfeiffer1987concept,shiller1992computation}. Here,
we define the Maximum and Minimum Acceleration Surfaces (MaAS and
MiAS), which are the third-order counterparts of the MVC.

\begin{definition}[MaAS/MiAS] Consider the set of feasible accelerations
\[
  \calF(s, \dot s) := \{\ddot s \;|\; \exists \dddot s, \vec a(s)
  \dddot s + \vec b(s) \dot s \ddot s + \vec c(s) \dot s^3 + \vec d(s)
  \leq 0\}.
\]
We define the MaAS and MiAS by
\begin{equation}
\label{eq:MaAS-MiAS}
\begin{aligned}
\MaAS(s, \dot s) &:= \max_{\ddot s \in \calF(s, \dot s)} \ddot s\\
\MiAS(s, \dot s) &:= \min_{\ddot s \in \calF(s, \dot s)} \ddot s. \\
\end{aligned}
\end{equation}
If $\calF(s, \dot s)$ is empty then the surfaces are not defined at
$(s, \dot s)$.
\end{definition}

This definition does not use the maximal controls $\gamma$ and $\eta$
but directly uses the constraints~(\ref{eq:third}). The advantage is
that even in the presence of a constraint $k$ such that
\mbox{$a_k(s)=0$}, the surfaces are still well-defined.  Additionally,
on both surfaces, the maximum and minimum jerks are equal almost
everywhere except on singular curves (see Prop.~\ref{prop:1} in the
Appendix).


\subsection{Characterizing third-order singularities}

\begin{definition}[Singular curve]
We say that the $k$-th constraint triggers a singularity at $s^*$ if
$a_k(s^*)=0$ and the set $C$ defined by
\begin{equation}
  \label{eq:feasible}
  \begin{aligned}
    C&:= \left\{ (s^*, \dot s, \ddot s) \ | \ \exists \dddot s : \right.\\
    &\left . 
    \begin{aligned}
      &a_k(s^*)\dddot s + b_k(s^*) \dot s \ddot s + c_k(s^*) \dot s^3 + d_k(s^*) = 0, \\
      &a_i(s^*)\dddot s + b_i(s^*) \dot s \ddot s + c_i(s^*) \dot
      s^3 + d_i(s^*) \leq 0, i \neq k \\
    \end{aligned} \right \}
  \end{aligned}
\end{equation}
is non-empty. We say $C$ is the singular curve at $s^*$.
\end{definition}

We show in the Appendix that all singular curves lie on either the
MaAS or MiAS (Prop.~\ref{prop:2}). Furthermore, a singular curve lies on the MaAS if
$b_k(s^*) > 0$ and on the MiAS if $b_k(s^*) < 0$. We shall also refer
to singular curves on MaAS and MiAS as maximum and minimum singular
curves respectively.

As in the second-order case, we can note the following behaviors:
\begin{itemize}
\item Forward integrations following maximum jerk and backward
  integrations following minimum jerk diverge when they approach a
  minimum singular curve;
\item Forward integrations following minimum jerk and backward
  integrations following maximum jerk diverge when they approach a
  maximum singular curve;
\item Forward and backward integrations following either maximum jerk
  or minimum jerk \emph{starting from} a feasible point at $s^*$ are
  not affected by the singular curves.
\end{itemize}
To understand these observations, we note that at 
fixed velocity $\dot s = \dot s_0$, constraints~(\ref{eq:feasible})
has the same form as second order-constraints
\begin{equation}
\label{eq:second}
  \vec a_{\mathrm{2nd}}(s) \ddot s + \vec b_{\mathrm{2nd}}(s) \dot s^2 
  + \vec c_{\mathrm{2nd}}(s) \leq 0,
\end{equation}
where $\vec a_{\mathrm{2nd}}, \vec b_{\mathrm{2nd}}, \vec
c_{\mathrm{2nd}}$ are the corresponding vectors. In the second-order
case, constraints~(\ref{eq:second}) cause integrations to diverge.
It therefore suggests that integrations
(third-order) projected on to a fixed velocity surface would likely
diverge, which is consistent with our observations. A more rigorous
analysis is left for future work.


To compute a singular curve $C$ at $s^{*}$, we simply find the maximum
and minimum velocities $\dot s^*_{\max}, \dot s^*_{\min}$ of each
curve and compute $\ddot s^{*}$ by the equality constraint in
Eq.~\eqref{eq:feasible}. This procedure correctly returns $C$ because
a singular curve is connected (See Prop.~\ref{prop:connected} in the
Appendix).

To compute $\dot s^*_{\max}, \dot s^*_{\min}$ one needs to solve
a pair of linear programming problems
\[
\begin{aligned}
    &\mathrm{maximize} \; & & [0, 0, 1]^T [\dddot s, \dot s \ddot s, \dot s^3] \\
    &\mathrm{subject\;to} \; & & (\ref{eq:feasible})\\
    \text{and   }&\\
    &\mathrm{maximize} \; & & [0, 0, -1]^T [\dddot s, \dot s \ddot s, \dot s^3] \\
    &\mathrm{subject\;to} \; & & (\ref{eq:feasible}).\\
\end{aligned}
\]



\subsection{Extending a profile through a singularity}

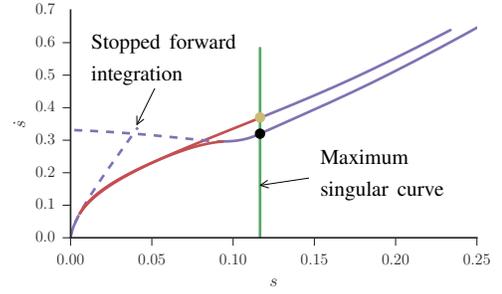
\begin{figure}[ht]
  \centering
  \begin{tikzpicture}
    \node[anchor=south west,inner sep=0] (image) at (0,0)
    {\includegraphics[width=0.35
      \textwidth]{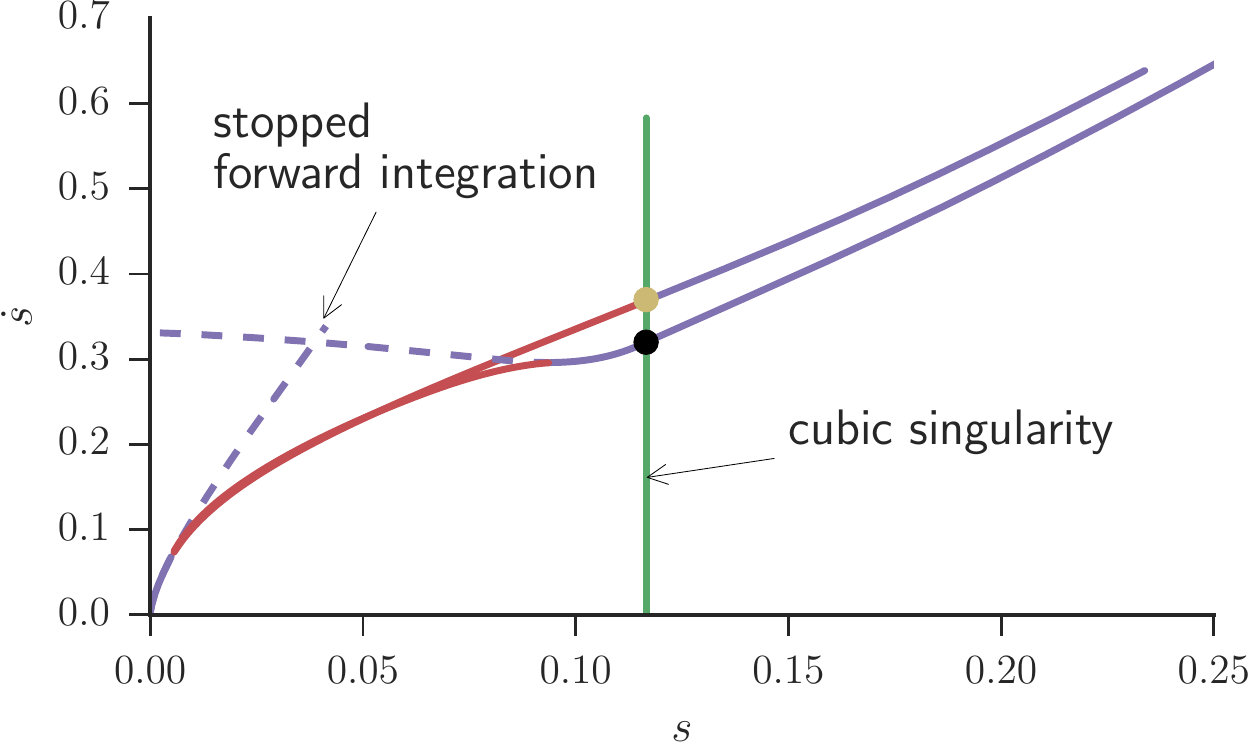}};
    \begin{scope}[x={(image.south east)},y={(image.north west)}]
      \fill [fill=white] (0.15,0.7) rectangle (0.5,0.9);
      \fill [fill=white] (0.62,0.3) rectangle (0.9,0.5);
      \node[text width=2cm] at (0.8, 0.4) {\footnotesize Maximum singular curve};
      \node[text width=3cm] at (0.4, 0.8) {\footnotesize Stopped forward integration};
    \end{scope}
  \end{tikzpicture}
  \caption{ \label{fig:comparing-extension-type-I} Comparison of the
    conjectured time-optimal extension (via the yellow point) with an
    extension computed by picking a feasible point lying at $s^*$
    (black point) and uses it to extend the forward integration. Note
    that the profiles are projected onto the $(s, \dot s)$ plane.  }
\end{figure}

We now describe a method, termed \extend, to extend a forward maximum
jerk profile through a singularity $s^*$. As noted before, integrating
forward following maximum jerk \emph{from} $s^*$ is not problematic,
so the main difficulty consists of connecting backwards to the initial
forward profile.

Note first that there exists a \emph{family} of possible backward
connections, obtained as below
\begin{enumerate}
  \item choose a feasible point at $s^*$ (i.e. a point $(s^*,\dot
    s,\ddot s)$ for which there exists a feasible jerk);
  \item integrate backward following maximum jerk from this point;
  \item connect the original forward maximum jerk profile to the new
    backward maximum jerk profile by the \bridge procedure of
    Section~\ref{sec:connection}.
\end{enumerate}

We conjecture that, within this family of possible connections, the
\emph{optimal} connection has the following properties
\begin{enumerate}
\item the starting point at $s^*$ belongs to the \emph{maximum
  singular curve};
\item there is no backward maximum jerk profile: in other words, the
  connecting minimum jerk profile starts directly at $s^*$;
\item as a consequence, the min-to-max switch happens directly at
  $s^*$ and not before, as in the generic case above.
\end{enumerate}
See Fig.~\ref{fig:comparing-extension-type-I} for an illustration.

Accordingly, we propose the following strategy to perform the backward
connection using MSM. Using subscript $\A$ and $\C$ to denote the
forward profile and the singular curve, we define a solution
$\vec x \in \bbR^{2N+4}$ as
\[
  \vec x:=[\dot s_0, \ddot s_0,...,\dot s_N, \ddot s_N, s_{\A}, \dot s_{\C}], 
\]
where $\dot s_i$ and $\ddot s_i$ are the velocity and acceleration at
$s_i$, $s_{\A}$ and $\dot s_{\C}$ are the guessed starting position
and ending velocity on $\A$ and $\C$ respectively. This gives the
defect function
\begin{equation}
    \label{eq:sglr}
    F(\vec x) := 
    \begin{pmatrix}
        X(s_1, \dot s_1, \ddot s_1, s_0) - [\dot s_0, \ddot s_0]^T \\
        X(s_2, \dot s_2, \ddot s_2, s_1) - [\dot s_1, \ddot s_1]^T  \\
        ...\\
        X(s_{N}, \dot s_{N}, \ddot s_{N}, s_{N-1}) - [\dot s_{N-1}, \ddot s_{N-1}]^T \\
        r_{\A}(s_A) - [\dot s_0, \ddot s_0]^T \\
        \bar r_{\C}(\dot s_{\C}) - [\dot s_N, \ddot s_N]^T\\        
    \end{pmatrix},
\end{equation}
where $r_{\A}(s_{\A})$ and $\bar r_{\C}(s_{\C})$ give the velocity and
acceleration on the profile and the curve. Again, Newton's method can
be employed to find the root of~(\ref{eq:sglr}).

There exists an additional difficulty with respect to Section~III: the
optimal jerk is ill-defined on the singular curves because of a
division by zero: $a_k(s^*)=0$. Taking again inspiration
from~\cite{Pha14tro}, we can show that the optimal jerk on the
singular curves is in fact given by the following \emph{singular jerk}
\begin{equation}
    \label{eq:singular-jerk}
    \dddot s_\mathrm{sglr} = - \frac
    {d_k'(s^*) + c_k'(s^*) \dot s^3 + 3c_k(s^*)\dot s \ddot s + b(s^*) (\dot s \ddot s + \ddot s ^2 / \dot s)}
    {a_k'(s^*) + b_k(s^*)}.
\end{equation}
This expression can be derived similarly as
in~\cite{Pha14tro}.

Fig.~(\ref{fig:comparing-extension-type-I}) compares the conjectured
optimal connection with a generic connection. One can note that the
conjectured profile has a higher velocity at \emph{any position},
which implies time-optimality.

The case of a backward extension can be treated similarly.
\section{Numerical experiments}
\label{sec:experiments}

\subsection{Simulation results}

We implemented and tested TOPP3 on a random geometric path subjecting
to constraints on joint velocities, accelerations and jerks
(Table~\ref{tab:limits}). Implementation of acceleration and velocity
constraints follows~\cite{Pha14tro}.

\begin{table}
\caption{Kinematic limits for Denso robot arm}
\label{tab:limits}
\begin{tabular}{lrrrrrr}
\toprule
Limits                      & J1    & J2    & J3    & J4    & J5    & J6    \\
\midrule
Vel $(\SI{}{rad s^{-1}})$          & 3.92  & 2.61  & 2.85  & 3.92  & 3.02  & 6.58  \\
Accel $(\SI{}{rad s^{-2}})$     & 19.7 & 16.8 & 20.7 & 20.9 & 23.7 & 33.5 \\ 
Jerk  $(\SI{}{rad s^{-3}})$           & 100.   & 100.   & 100.   & 100.   & 100.   & 100.   \\ 
\bottomrule
\end{tabular}
\end{table}

Different scenarios were considered.  In the first one, we set
restrictive bounds on joint jerks at $\SI{100}{rad/s^3}$.  There were
several singularities which are plotted as green lines in
Fig.~\ref{fig:resulting-profile}.  Despite the existence of these
singularities, TOPP3 was able to compute the time-optimal
parametrization.  We note that TOPP3 extended the forward profile once
via the maximum singular curve on the left.

\begin{figure}[tb]
  \centering
  \begin{tikzpicture}
    \node[anchor=south west,inner sep=0] at (0,0)
         {\includegraphics[width=0.45\textwidth]{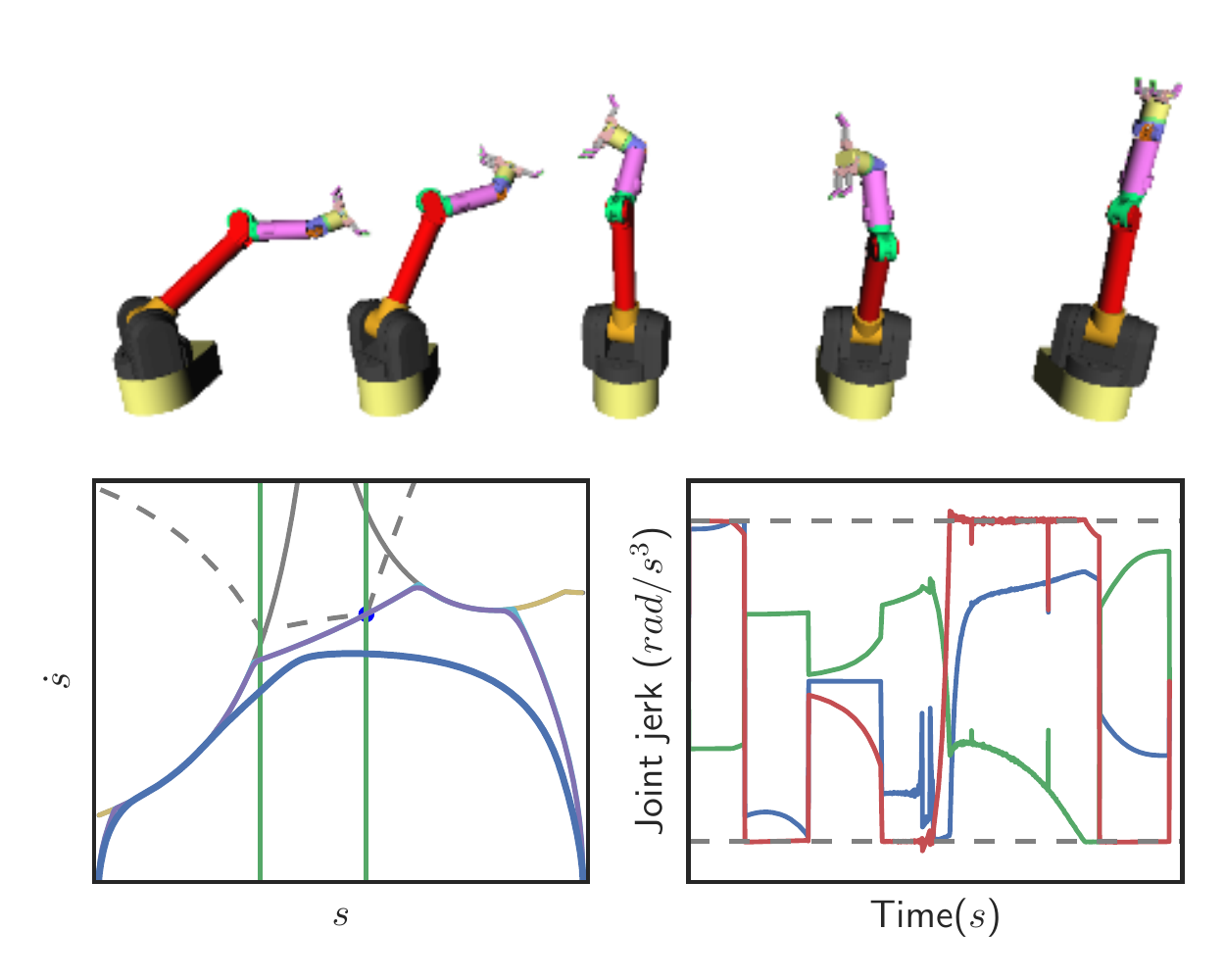}}; 
    \node at (1, 6) {\textbf{A}};
    \node at (1, 3.5) {\textbf{B}};
    \node at (5., 3.5) {\textbf{C}};
  \end{tikzpicture}
  \caption{ \textbf{A}: Snapshots of the robot.  \textbf{B}: The
    profile constrained at $\SI{100}{rad/s^3}$ (blue), the profile
    constrained at $\SI{1000}{rad/s^3}$ (purple) and the profile with
    unconstrained jerk (light blue). Singularities are plotted as
    green lines. Note that the profiles are projected onto the
    $(s, \dot s)$ plane.  \textbf{C}: Joint jerks versus time.}
    \label{fig:resulting-profile}
\end{figure}
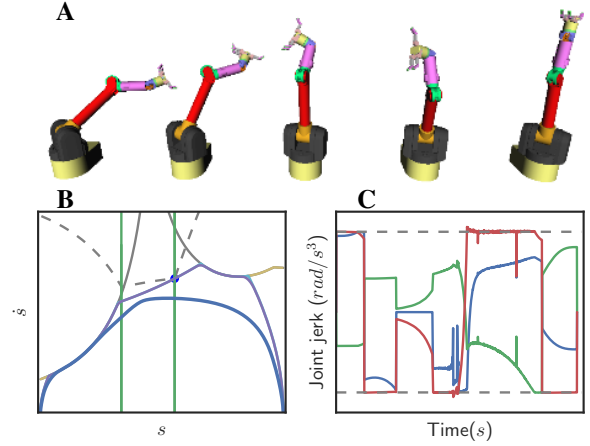

Next, we considered a more practical set of bounds at
$\SI{1000}{rad/s^3}$ on joint jerks.  In this case, TOPP3 also
terminated successfully. Moreover, we found that singularities did not
affect the profiles: the final profile was computed without having to
extend any profile via any singularities. 

Lastly, we compared these profiles with one that does not subject to
bounds on joint jerks (computed with the original TOPP algorithm).  We
observed that the profile bounded at $\SI{100}{rad/s^3}$ appears to be
a smoothed version of the profile bounded at $\SI{1000}{rad/s^3}$,
which in turn is a smoothed version of the one not subject to any jerk
bound (Fig.~\ref{fig:resulting-profile}).

Fig.~\ref{fig:resulting-trajectory} shows the joint trajectories
computed from these profiles. We note that the profiles are bang-bang,
satisfy all kinematic constraints and have different total duration.
In particular, they last $\SI{1.36}{sec}$, $\SI{1.19}{sec}$ and
$\SI{1.18}{sec}$. Note that, when there is no bound on jerk, the
latter reaches extremely high values
(Fig.~\ref{fig:resulting-trajectory}-C3).

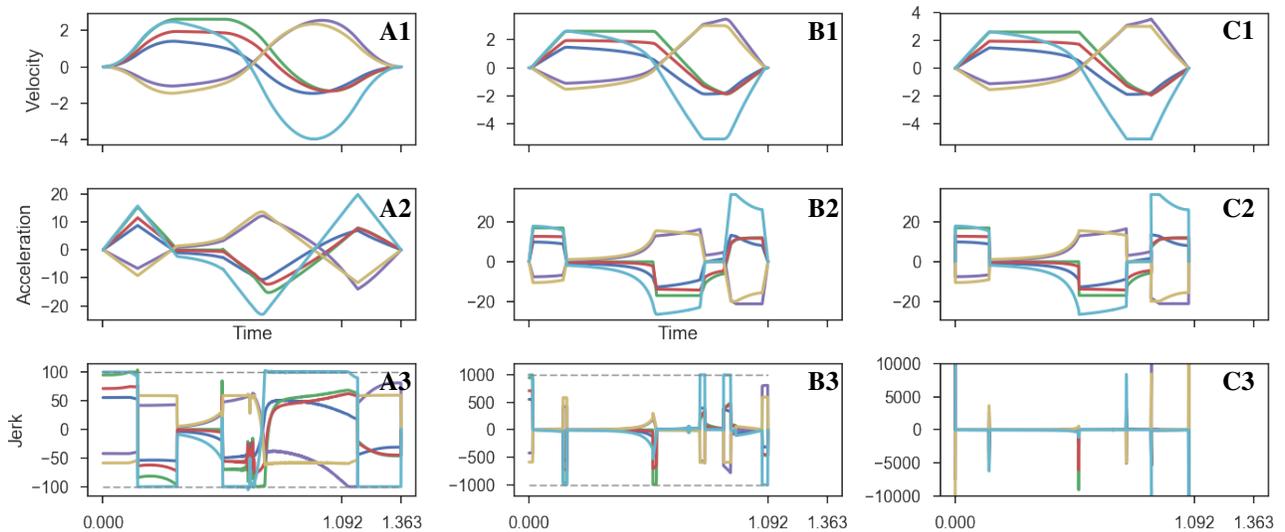
\begin{figure*}[tb]
    \centering
    \begin{tikzpicture}
    \node[anchor=south west,inner sep=0] (image) at (0,0)
    {\includegraphics[width=0.95\textwidth]{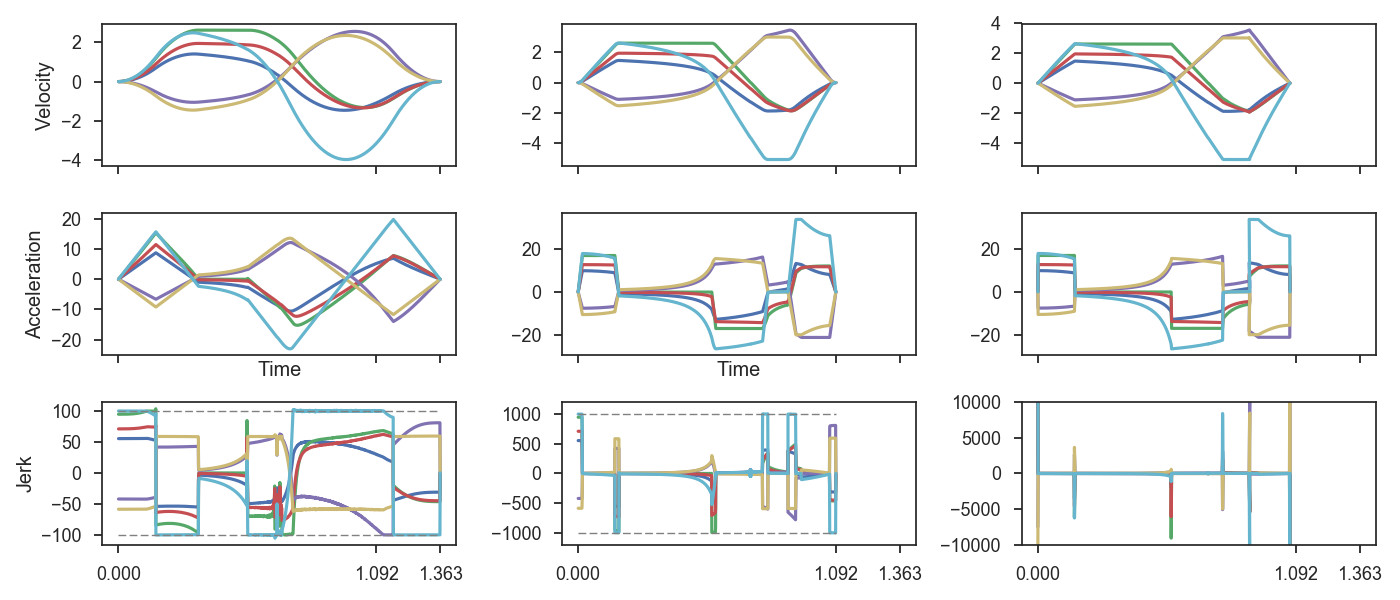}};
    \begin{scope}[x={(image.south east)},y={(image.north west)}]
      \node at (0.31, 0.93) {\textbf{A1}}; 
      \node at (0.31, 0.61) {\textbf{A2}}; 
      \node at (0.31, 0.3) {\textbf{A3}}; 
      \node at (0.64, 0.93) {\textbf{B1}}; 
      \node at (0.64, 0.61) {\textbf{B2}}; 
      \node at (0.64, 0.3) {\textbf{B3}}; 
      \node at (0.96, 0.93) {\textbf{C1}}; 
      \node at (0.96, 0.61) {\textbf{C2}}; 
      \node at (0.96, 0.3) {\textbf{C3}}; 
    \end{scope}
    \end{tikzpicture}
    \caption{Comparing the resulting trajectories from TOPP3 with
      $\SI{100}{rad/s^3}$ bound on joint jerks (\textbf{A1-3}),
      $\SI{1000}{rad/s^3}$ bound on joint jerks (\textbf{B1-3}) and
      one from TOPP without any bound on jerk (\textbf{C1-3}). Note
      that, when there is no bound on jerk, the latter reaches
      extremely high values (C3).  }
    \label{fig:resulting-trajectory}
\end{figure*}

Table~\ref{tab:computation-time} reports computation time for the
three scenarios respectively. The experiments were ran on a single
core at $\SI{2.00}{GHz}$ and $\SI{8}{Gb}$ of memory. We remark that
the pre-processing -- computation of singularities and switch points
-- is written in Python and is significantly slower than the integration
and MSM procedures which are written in Cython.  Therefore, we
only compared the total time$^*$ which neglects pre-processing. In
this regard, computing the first and second profiles took 83 times and
19 times longer than the third one.

\begin{table}
\caption{Computation time for different levels of bound on joint jerks}
\label{tab:computation-time}
\centering
\begin{tabular}{llrrrrr}
    \toprule
  \multicolumn{2}{l}{Jerk bound $(\SI{}{rad/s^{-3}})$} & $100$	     & $1000$ & $\infty$    \\
    \midrule
  \multicolumn{2}{l}{Time step $(\SI{}{ms})$}          & 1           & 1      & 1           \\
  \multicolumn{2}{l}{Total $(\SI{}{ms})$}              & 754         & 181    & 72          \\
                                                       & \integrate  & 4.05   & 2.36 & 2.24 \\ 
                                                       & \bridge     & 364    & 42.2 & 0.3  \\ 
                                                       & \extend     & 134    & 73.2 & 3.50 \\
                                                       & Pre-process & 243    & 63.7 & 65.9 \\
  \multicolumn{2}{l}{Total*\footnotemark $(\SI{}{ms})$ }             
                                                       & 511         & 118    & 6.1         \\
                                                       &             & 83x    & 19x  & 1x   \\
    \bottomrule
\end{tabular}
\end{table}

\subsection{Singularities caused by third-order constraints}

Third-order singularities appear quite frequently, almost as frequent
as singular switch points which are many~\cite{Pha14tro}.  In the
Fig.~\ref{fig:resulting-profile}, one can see that the singular curves
(green lines) correspond to points where the MVC is continuous but is
not differentiable, which are singular switch points~\cite{Pha14tro}.


We observed that in the second scenario ($\SI{1000}{rad/s^3}$),
singularities did not affect integrations.  We found that this
phenomenon only appears when the bounds on joint jerks are high in
comparison with the bounds on joint accelerations.  Here is one
possible explanation: at high jerk bounds, the singular curves have
high acceleration.  However, as second-order constraints restrict
integrations from having high acceleration, the integrations will not
come close to these curves. Now, as analyzed, we observe that
integrations only diverge near to the singular curves; therefore,
these integrations are not affected.  The precise conditions at which
this happens is left for future work.
\footnotetext{Not including time to compute singularities and switch points.} 


\section{Conclusion}
\label{sec:discussion}

In this paper, we have studied the structure of the Time-Optimal Path
Parameterization (TOPP) problem with third-order constraints. We have
argued that the optimal profile can be obtained by integrating
alternatively maximum and minimum jerk, yielding a
max-min-max-\dots-min-max structure. We have identified two main
difficulties in the integration process: (i) how to smoothly connect
two maximum jerk profiles by a minimum jerk profile, and (ii) how to
extend the integration through singularities which, if not properly
addressed, would systematically cause integration failure. We proposed
some solutions to these two difficulties and, based on these
solutions, implemented an algorithm -- TOPP3 -- to solve the TOPP
problem with third-order constraints in a number of representative
scenarios.

There are still a number of open theoretical and practical questions,
which we are currently actively investigating. For instance,
\begin{enumerate}
\item Under which condition the profile returned by the connection
  procedure (\bridge) is unique? Alternatively, can we enumerate all
  connecting profiles?
\item Similarly, under which condition the profile returned by the
  extension procedure (\extend) is unique?
\item How to characterize and address the tangent switch points (which
  probably exist in the form of tangent \emph{curves})?
\end{enumerate}

Addressing all these (difficult) questions will enable us to implement
TOPP3 in a fast and robust manner, which in turn can be useful for a
wide range of robotics applications.

\appendix
\subsection{Some proofs regarding third-order singularities}

\begin{proposition}
  \label{prop:2}
  All singular curve lies on either the MaAS or the MiAS. Furthermore,
  a singular curve lies on the MaAS if $b_k(s^*) > 0$ and on the MiAS
  if $b_k(s^*) < 0$.
\end{proposition}
We assume $b_k(s^*) \neq 0$ since it is fairly rare for both
$a_k$ and $b_k$ to become zero.
\begin{proof}
  We will prove the first proposition by contradiction. Consider a
  singular curve $C$ triggered by the $k$-th constraint, there exists
  a point $(s^*, \dot s^*, \ddot s^*) \in C$ that do not lie on either
  the MaAS or MiAS. That is
  \[
    \MiAS(s^*, \dot s^*) < \ddot s^* < \MaAS(s^*, \dot
    s^*).
  \]
  It follows that we can always find $\ddot s_1$ and $\ddot s_2$ such
  that
  \[
    \MiAS(s^*, \dot s^*) < \ddot s_2 < \ddot s^* < \ddot s_1
    <\MaAS(s^*, \dot s^*).
  \]
  Now, by equation~(\ref{eq:feasible}), we have the equality
  \begin{equation}
    \label{eq:equality2}
    b_k(s^*)\dot s^* \ddot s^* + c_k(s^*)\dot s^{*3} + d_k(s^*) = 0,
  \end{equation}
  since $a_k(s^{*}) = 0$.  Now, if $b_k(s^*) > 0$, we replace
  $\ddot s^*$ with $\ddot s_1$ and notice that $\dot s > 0$ to see
  that equation~(\ref{eq:equality2}) is strictly positive. Similarly
  if $b_k(s^*) < 0$ then we replace $\ddot s^*$ with $\ddot s_2$ to
  see that equation~\eqref{eq:equality2} is strictly negative. Both
  are contradictions.  We neglect the case where $b_k(s^*)=0$.

  Finally, to prove the second proposition, we simply remark that if
  $b_k(s^*) > 0$, then there must not exists any feasible $\ddot s_1$
  that is greater than $\ddot s^*$. It therefore follows that
  \[
    \ddot s^* = \MaAS(s^*, \dot s^*).
  \]
  The case where $b_k(s^*) < 0$ is proven similarly.
\end{proof}

\begin{proposition}
  \label{prop:1}
The maximum and minimum jerks are equal almost everywhere on the MaAS and the MiAS
except on singular curves.
\end{proposition}
\begin{proof}
Considering a point at $(s_0, \dot s_0, \ddot s_0)$ on the MaAS which does not 
lie on any singular curve, by definition, $\ddot s_0$ is the 
maximum acceleration for $s = s_0, \dot s=\dot s_0$ subjecting to
constraints~(\ref{eq:feasible}).
Since all constraints~(\ref{eq:feasible}) are linear the set of feasible
$(\ddot s, \dddot s)$ is a convex polygon on the plane.

Now, since this polygon does not contain any edge that is 
parallel to the $\ddot s$-axis ($(s_0, \dot s_0, \ddot s_0)$ does not lie on 
any singular curve), $\ddot s=\ddot s_0$ is maximized at a 
single vertex of the polygon. Therefore there is only one 
feasible jerk for $\ddot s = \ddot s_0$. 

A similar proof can be written for the MiAS.
\end{proof}

\begin{proposition}
\label{prop:connected}
Singular curves are connected.
\end{proposition}

\begin{proof}
Consider a singular curve $C$ in the $(s, \dot s, \ddot s)$ space triggered
by the $k$-th constraint. Let the convex set defined by the feasibility 
condition~(\ref{eq:feasible}) be $\calA$ and the mapping from $\calA$ using 
the equality 
$$b_k(s^*) \dot s \ddot s + c_k(s^*) \dot s^3 + d_k(s^*) = 0.$$
to the singular curve $C$ be $f$. By definition, $f(\calA)=C$.

Now, since $\calA$ is convex and thus connected and $f$ is also continuous
for $b_k(s^*)\neq 0, \dot s^*\neq 0$, using Theorem~4.22 in~\cite{rudin1964principles}, 
it follows that $C$ is connected.
\end{proof}


\bibliographystyle{IEEEtran}
\bibliography{library}
\end{document}